\documentclass{article}
\usepackage{proceed2e} 

\usepackage[english]{babel}
\usepackage[utf8]{inputenc}

\usepackage[unicode]{hyperref}
\usepackage{enumerate}
\usepackage{times}
\usepackage{amsmath}
\usepackage{amssymb}
\usepackage{booktabs}
\usepackage{xtab} 
\usepackage{paralist} 
\usepackage{tikz}

\usepackage{mythm}
\usepackage{aixi}

\usepackage{algorithm}
\usepackage[noend]{algpseudocode}

\newcommand{\assref}[1]{\hyperref[#1]{\autoref*{ass:grl}\ref*{#1}}}

\def\Mrefl{\M_{\textrm{refl}}}

\usepackage{etoolbox}
\AtEndEnvironment{example}{
	\qed
}

\usepackage[explicit]{titlesec}
\titleformat{\section}{\large\bf\raggedright}{\thesection}{1em}{\MakeUppercase{#1}}
\titleformat{\subsection}{\normalsize\bf\raggedright}{\thesubsection}{1em}{\MakeUppercase{#1}}

\AtBeginDocument{

}

\title{A Formal Solution to the Grain of Truth Problem}
\author{
	{\bf Jan Leike} \\
	Australian National University \\
	\href{mailto:jan.leike@anu.edu.au}{jan.leike@anu.edu.au}
	\And
	{\bf Jessica Taylor} \\
	Machine Intelligence Research Inst. \\
	\href{mailto:jessica@intelligence.org}{jessica@intelligence.org}
	\And
	{\bf Benya Fallenstein} \\
	Machine Intelligence Research Inst. \\
	\href{mailto:benya@intelligence.org}{benya@intelligence.org}
}

\begin{document}

\maketitle

\begin{abstract}%
A Bayesian agent acting in a multi-agent environment
learns to predict the other agents' policies
if its prior assigns positive probability to them
(in other words, its prior contains a \emph{grain of truth}).
Finding a reasonably large class of policies
that contains the Bayes-optimal policies with respect to this class
is known as the \emph{grain of truth problem}.
Only small classes are known to have a grain of truth
and the literature contains several related impossibility results.
In this paper we present a formal and general solution to
the full grain of truth problem:
we construct a class of policies that contains all computable policies
as well as
Bayes-optimal policies for every lower semicomputable prior over the class.
When the environment is unknown,
Bayes-optimal agents may fail to act optimally even asymptotically.
However, agents based on Thompson sampling
converge to play $\varepsilon$-Nash equilibria
in arbitrary unknown computable multi-agent environments.
While these results are purely theoretical,
we show that they can be computationally approximated
arbitrarily closely.
\end{abstract}


\section{Introduction}
\label{sec:introduction}

Consider the general setup of multiple reinforcement learning agents
interacting sequentially
in a known environment with the goal to maximize discounted reward.%
\footnote{We mostly use the terminology of reinforcement learning.
For readers from game theory we provide a dictionary in
\autoref{tab:rl-game-theory-translation}.}
Each agent knows how the environment behaves,
but does not know the other agents' behavior.
The natural (Bayesian) approach would be to define a class of possible
policies that the other agents could adopt and take a prior over this class.
During the interaction, this prior gets updated to the posterior
as our agent learns the others' behavior.
Our agent then acts optimally with respect to this posterior belief.

\begin{table}[t]
\begin{center}
\begin{tabular}{p{0.445\columnwidth}p{0.445\columnwidth}}
\toprule
Reinforcement learning & Game theory \\
\midrule
stochastic policy & mixed strategy \\
deterministic policy & pure strategy \\
agent & player \\
multi-agent environment & infinite extensive-form game \\
reward & payoff/utility \\
(finite) history & history \\
infinite history & path of play \\
\bottomrule
\end{tabular}
\end{center}
\caption{Terminology dictionary between reinforcement learning and game theory.}
\label{tab:rl-game-theory-translation}
\end{table}

A famous result for infinitely repeated games states that
as long as each agent assigns positive prior probability
to the other agents' policies (a \emph{grain of truth})
and each agent acts Bayes-optimal,
then the agents converge to playing an $\varepsilon$-Nash equilibrium~\cite{KL:1993}.

As an example,
consider an infinitely repeated prisoners dilemma between two agents.
In every time step the payoff matrix is as follows,
where C means cooperate and D means defect.
\begin{center}
\begin{tabular}{l|cc}
  & C        & D \\
\hline
C & 3/4, 3/4 & 0, 1 \\
D & 1, 0     & 1/4, 1/4
\end{tabular}
\end{center}
Define the set of policies $\Pi := \{ \pi_\infty, \pi_0, \pi_1, \ldots \}$
where policy $\pi_t$ cooperates until time step $t$ or
the opponent defects (whatever happens first) and defects thereafter.
The Bayes-optimal behavior is to cooperate until the posterior belief that
the other agent defects in the time step after the next
is greater than some constant (depending on the discount function)
and then defect afterwards.
Therefore Bayes-optimal behavior leads to a policy from the set $\Pi$
(regardless of the prior).
If both agents are Bayes-optimal with respect to some prior,
they both have a grain of truth and therefore they converge to
a Nash equilibrium:
either they both cooperate forever or
after some finite time they both defect forever.
Alternating strategies like TitForTat
(cooperate first, then play the opponent's last action)
are not part of the policy class $\Pi$,
and adding them to the class breaks the grain of truth property:
the Bayes-optimal behavior is no longer in the class.
This is rather typical;
a Bayesian agent usually needs to be
more powerful than its environment~\cite{LH:2015computability}.

Until now, classes that admit a grain of truth
were known only for small toy examples such as
the iterated prisoner's dilemma above~\cite[Ch.~7.3]{SLB:2009}.
The quest to find a large class admitting a grain of truth is known as
the \emph{grain of truth problem}~\cite[Q.~5j]{Hutter:2009open}.
The literature contains
several impossibility results on the grain of truth problem%
~\cite{FY:2001impossibility,Nachbar:1997,Nachbar:2005}
that identify properties that cannot be simultaneously
satisfied for classes that allow a grain of truth.

In this paper we present a formal solution to multi-agent reinforcement learning
and the grain of truth problem
in the general setting (\autoref{sec:a-grain-of-truth}).
We assume that our multi-agent environment is computable,
but it does not need to be stationary/Markov, ergodic, or finite-state~\cite{Hutter:2005}.
Our class of policies is large enough to contain all computable (stochastic) policies,
as well as all relevant Bayes-optimal policies.
At the same time, our class is small enough to be limit computable.
This is important because
it allows our result to be computationally approximated.

In \autoref{sec:multi-agent-environments}
we consider the setting where the multi-agent environment is
unknown to the agents and has to be learned
in addition to the other agents' behavior.
A Bayes-optimal agent may not learn to act optimally
in unknown multi-agent environments \emph{even though it has a grain of truth}.
This effect occurs in non-recoverable environments where
taking one wrong action can mean a permanent loss of future value.
In this case, a Bayes-optimal agent avoids taking these dangerous actions
and therefore will not explore enough
to wash out the prior's bias~\cite{LH:2015priors}.
Therefore, Bayesian agents are not \emph{asymptotically optimal}, i.e.,
they do not always learn to act optimally~\cite{Orseau:2013}.

However, asymptotic optimality is achieved by Thompson sampling
because the inherent randomness of Thompson sampling
leads to enough exploration to learn the entire environment class%
~\cite{LLOH:2016Thompson}.
This leads to our main result:
if all agents use Thompson sampling over our class of multi-agent environments,
then for every $\varepsilon > 0$
they converge to an $\varepsilon$-Nash equilibrium asymptotically.

The central idea to our construction is based on
\emph{reflective oracles}~\cite{FST:2015,FTC:2015reflection}.
Reflective oracles are probabilistic oracles
similar to halting oracles that
answer whether the probability that
a given probabilistic Turing machine $T$ outputs $1$
is higher than a given rational number $p$.
The oracles are reflective in the sense that the machine $T$
may itself query the oracle,
so the oracle has to answer queries about itself.
This invites issues caused by self-referential liar paradoxes of the form
``if the oracle says that I return $1$ with probability $> 1/2$,
then return $0$, else return $1$.''
Reflective oracles avoid these issues by being allowed to randomize if
the machines do not halt or the rational number
is \emph{exactly} the probability to output $1$.
We introduce reflective oracles formally in \autoref{sec:reflective-oracles}
and prove that there is a limit computable reflective oracle.

\section{Reflective Oracles}
\label{sec:reflective-oracles}

\subsection{Preliminaries}
\label{ssec:preliminaries}

Let $\X$ denote a finite set called \emph{alphabet}.
The set $\X^* := \bigcup_{n=0}^\infty \X^n$ is
the set of all finite strings over the alphabet $\X$,
the set $\X^\infty$ is
the set of all infinite strings
over the alphabet $\X$, and
the set $\X^\sharp := \X^* \cup \X^\infty$ is their union.
The empty string is denoted by $\epsilon$, not to be confused
with the small positive real number $\varepsilon$.
Given a string $x \in \X^\sharp$, we denote its length by $|x|$.
For a (finite or infinite) string $x$ of length $\geq k$,
we denote with $x_{1:k}$ the first $k$ characters of $x$,
and with $x_{<k}$ the first $k - 1$ characters of $x$.
The notation $x_{1:\infty}$ stresses that $x$ is an infinite string.

A function $f: \X^* \to \mathbb{R}$ is
\emph{lower semicomputable} iff
the set $\{ (x, p) \in \X^* \times \mathbb{Q} \mid f(x) > p \}$
is recursively enumerable.
The function $f$ is \emph{computable} iff
both $f$ and $-f$ are lower semicomputable.
Finally, the function $f$ is \emph{limit computable} iff
there is a computable function $\phi$ such that
\[
\lim_{k \to \infty} \phi(x, k) = f(x).
\]
The program $\phi$ that limit computes $f$
can be thought of as an \emph{anytime algorithm} for $f$:
we can stop $\phi$ at any time $k$ and get a preliminary answer.
If the program $\phi$ ran long enough (which we do not know),
this preliminary answer will be close to the correct one.

We use $\Delta\mathcal{Y}$ to denote
the set of probability distributions over $\mathcal{Y}$.
A list of notation can be found in \autoref{app:notation}.

\subsection{Definition}
\label{ssec:reflective-oracles-def}

A \emph{semimeasure} over the alphabet $\X$ is
a function $\nu: \X^* \to [0,1]$ such that
\begin{inparaenum}[(i)]
\item $\nu(\epsilon) \leq 1$, and
\item $\nu(x) \geq \sum_{a \in \X} \nu(xa)$ for all $x \in \X^*$.
\end{inparaenum}
In the terminology of measure theory,
semimeasures are probability measures on the probability space
$\X^\sharp = \X^* \cup X^\infty$
whose $\sigma$-algebra is generated by the \emph{cylinder sets}
$\Gamma_x := \{ xz \mid z \in \X^\sharp \}$
~\cite[Ch.\ 4.2]{LV:2008}.
We call a semimeasure (probability) a \emph{measure} iff
equalities hold in (i) and (ii) for all $x \in \X^*$.

Next, we connect semimeasures to Turing machines.
The literature uses \emph{monotone Turing machines},
which naturally correspond to
lower semicomputable semimeasures~\cite[Sec.\ 4.5.2]{LV:2008}
that describe the distribution that arises when piping fair coin flips
into the monotone machine.
Here we take a different route.

A \emph{probabilistic Turing machine} is a Turing machine that
has access to an unlimited number of uniformly random coin flips.
Let $\mathcal{T}$ denote the set of all probabilistic Turing machines
that take some input in $\X^*$ and may query an oracle (formally defined below).
We take a Turing machine $T \in \mathcal{T}$ to correspond to
a semimeasure $\lambda_T$ where
$\lambda_T(a \mid x)$ is the probability that
$T$ outputs $a \in \X$ when given $x \in \X^*$ as input.
The value of $\lambda_T(x)$ is then given by the chain rule
\begin{equation}\label{eq:chain-rule}
\lambda_T(x) := \prod_{k=1}^{|x|} \lambda_T(x_k \mid x_{<k}).
\end{equation}
Thus $\mathcal{T}$ gives rise to the set of semimeasures $\M$ where
the \emph{conditionals} $\lambda(a \mid x)$ are lower semicomputable.
In contrast, the literature typically considers semimeasures
whose \emph{joint} probability \eqref{eq:chain-rule} is lower semicomputable.
This set $\M$ contains all computable measures.
However, $\M$ is a proper subset of the set of all lower semicomputable semimeasures
because the product \eqref{eq:chain-rule} is lower semicomputale,
but there are some lower semicomputable semimeasures whose conditional
is not lower semicomputable~\cite[Thm.~6]{LH:2015computability2}.

In the following we assume that our alphabet is binary,
i.e., $\X := \{ 0, 1 \}$.

\begin{definition}[Oracle]
\label{def:oracle}
An \emph{oracle} is a function
$O: \mathcal{T} \times \{ 0, 1 \}^* \times \mathbb{Q} \to \Delta \{ 0, 1 \}$.
\end{definition}

Oracles are understood to be probabilistic:
they randomly return $0$ or $1$.
Let $T^O$ denote the machine $T \in \mathcal{T}$ when run with the oracle $O$,
and let $\lambda_T^O$ denote the semimeasure induced by $T^O$.
This means that drawing from $\lambda_T^O$ involves two sources of randomness:
one from the distribution induced by the probabilistic Turing machine $T$
and one from the oracle's answers.

The intended semantics of an oracle are that it takes
a \emph{query} $(T, x, p)$ and returns $1$ if the machine $T^O$ outputs $1$
on input $x$ with probability greater than $p$ when run with the oracle $O$,
i.e., when $\lambda^O_T(1 \mid x) > p$.
Furthermore, the oracle returns $0$ if the machine $T^O$ outputs $1$
on input $x$ with probability less than $p$ when run with the oracle $O$,
i.e., when $\lambda^O_T(1 \mid x) < p$.
To fulfill this,
the oracle $O$ has to make statements about itself,
since the machine $T$ from the query may again query $O$.
Therefore we call oracles of this kind \emph{reflective oracles}.
This has to be defined very carefully
to avoid the obvious diagonalization issues that are caused by programs
that ask the oracle about themselves.
We impose the following self-consistency constraint.

\begin{definition}[Reflective Oracle]
\label{def:reflective-oracle}
An oracle $O$ is \emph{reflective} iff
for all queries
$(T, x, p) \in \mathcal{T} \times \{ 0, 1 \}^* \times \mathbb{Q}$,
\begin{enumerate}[(i)]
\item $\lambda_T^O(1 \mid x) > p$ implies $O(T, x, p) = 1$, and
\item $\lambda_T^O(0 \mid x) > 1 - p$ implies $O(T, x, p) = 0$.
\end{enumerate}
\end{definition}

If $p$ under- or overshoots the true probability of $\lambda_T^O(\,\cdot \mid x)$,
then the oracle must reveal this information.
However, in the critical case when $p = \lambda_T^O(1 \mid x)$,
the oracle is allowed to return anything and may randomize its result.
Furthermore, since $T$ might not output any symbol,
it is possible that $\lambda_T^O(0 \mid x) + \lambda_T^O(1 \mid x) < 1$.
In this case the oracle can reassign the non-halting probability mass
to $0$, $1$, or randomize; see \autoref{fig:reflective-oracle}.

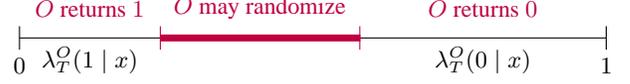
\begin{figure}[t]
\begin{center}
\small
\begin{tikzpicture}[scale=0.78]
\draw (0,0) to (10, 0);
\draw (0, .2) to (0, -.2) node[below] {$0$};
\draw (10, .2) to (10, -.2) node[below] {$1$};

\node[below] at (7.9, 0) {$\lambda_T^O(0 \mid x)$};
\node[below] at (1.2, 0) {$\lambda_T^O(1 \mid x)$};

\draw[purple] (5.8, .2) to (5.8, -.2);
\draw[purple] (2.4, .2) to (2.4, -.2);
\filldraw[purple] (2.4, .05) -- (5.8, .05) -- (5.8, -.05) -- (2.4, -.05);
\node[purple] at (1.2, 0.5) {$O$ returns $1$};
\node[purple] at (4.1, 0.5) {$O$ may randomize};
\node[purple] at (7.9, 0.5) {$O$ returns $0$};
\end{tikzpicture}
\end{center}
\caption{
Answer options of a reflective oracle $O$ for the query $(T, x, p)$;
the rational $p \in [0, 1]$ falls into one of the three regions above.
The values of $\lambda_T^O(0 \mid x)$ and $\lambda_T^O(1 \mid x)$ are depicted
as the length of the line segment under which they are written.
}\label{fig:reflective-oracle}
\end{figure}

\begin{example}[Reflective Oracles and Diagonalization]
\label{ex:diagonalization-for-reflective-oracles}
Let $T \in \mathcal{T}$ be a probabilistic Turing machine that
outputs $1 - O(T, \epsilon, 1/2)$
($T$ can know its own source code by quining~\cite[Thm.~27]{Kleene:1952}).
In other words, $T$ queries the oracle about whether it is more likely
to output $1$ or $0$, and then does whichever the oracle says is less likely.
In this case we can use an oracle $O(T, \epsilon, 1/2) := 1/2$
(answer $0$ or $1$ with equal probability),
which implies $\lambda_T^O(1 \mid \epsilon) = \lambda_T^O(0 \mid \epsilon) = 1/2$,
so the conditions of \autoref{def:reflective-oracle} are satisfied.
In fact, for this machine $T$ we must have
$O(T, \epsilon, 1/2) = 1/2$ for all reflective oracles $O$.
\end{example}

The following theorem establishes that reflective oracles exist.

\begin{theorem}[{\cite[App.\ B]{FTC:2015reflectionx}}]
\label{thm:existence-reflective-oracles}
There is a reflective oracle.
\end{theorem}

\begin{definition}[Reflective-Oracle-Computable]
\label{def:reflective-oracle-computable}
A semimeasure is called \emph{reflective-oracle-computable} iff
it is computable on a probabilistic Turing machine
with access to a reflective oracle.
\end{definition}

For any probabilistic Turing machine $T \in \mathcal{T}$
we can complete the semimeasure $\lambda_T^O(\,\cdot \mid x)$
into a reflective-oracle-computable measure
$\overline\lambda_T^O(\,\cdot \mid x)$:
Using the oracle $O$ and a binary search on the parameter $p$
we search for the crossover point $p$ where $O(T, x, p)$
goes from returning $1$ to returning $0$.
The limit point $p^* \in \mathbb{R}$ of the binary search is random
since the oracle's answers may be random.
But the main point is that the expectation of $p^*$ exists,
so $\overline\lambda_T^O(1 \mid x) = \mathbb{E}[p^*]
= 1 - \overline\lambda_T^O(0 \mid x)$ for all $x \in \X^*$.
Hence $\overline\lambda_T^O$ is a measure.
Moreover, if the oracle is reflective,
then $\overline\lambda_T^O(x) \geq \lambda_T^O(x)$ for all $x \in \X^*$.
In this sense the oracle $O$ can be viewed as
a way of `completing' all semimeasures $\lambda_T^O$
to measures by arbitrarily assigning the non-halting probability mass.
If the oracle $O$ is reflective this is consistent in the sense that
Turing machines who run other Turing machines will be completed in the same way.
This is especially important for a universal machine
that runs all other Turing machines
to induce a Solomonoff-style distribution.

\subsection{A Limit Computable Reflective Oracle}
\label{ssec:lc-reflective-oracle}

The proof of \autoref{thm:existence-reflective-oracles}
given in \cite[App.\ B]{FTC:2015reflectionx}
is nonconstructive and uses the axiom of choice.
In \autoref{ssec:lc-reflective-oracle-proof}
we give a constructive proof for the existence of reflective oracles
and show that there is one that is limit computable.

\begin{theorem}[A Limit Computable Reflective Oracle]
\label{thm:lc-reflective-oracle}
There is a reflective oracle that is limit computable.
\end{theorem}

This theorem has the immediate consequence
that reflective oracles cannot be used as halting oracles.
At first, this result may seem surprising:
according to the definition of reflective oracles,
they make concrete statements about the output of probabilistic Turing machines.
However, the fact that the oracles may randomize some of the time
actually removes enough information such that halting can no longer be decided
from the oracle output.

\begin{corollary}[Reflective Oracles are not Halting Oracles]
\label{cor:reflective-oracles-not-halting-oracles}
There is no probabilistic Turing machine $T$ such that
for every prefix program $p$ and every reflective oracle $O$,
we have that $\lambda_T^O(1 \mid p) > 1/2$ if $p$ halts and
$\lambda_T^O(1 \mid p) < 1/2$ otherwise.
\end{corollary}
\begin{proof}
Assume there was such a machine $T$ and
let $O$ be the limit computable oracle from \autoref{thm:lc-reflective-oracle}.
Since $O$ is reflective we can turn $T$ into a deterministic halting oracle
by calling $O(T, p, 1/2)$ which deterministically returns
$1$ if $p$ halts and $0$ otherwise.
Since $O$ is limit computable,
we can finitely compute the output of $O$ on any query
to arbitrary finite precision using our deterministic halting oracle.
We construct a probabilistic Turing machine $T'$ that uses our halting oracle
to compute (rather than query) the oracle $O$
on $(T', \epsilon, 1/2)$ to a precision of $1/3$ in finite time.
If $O(T', \epsilon, 1/2) \pm 1/3 > 1/2$, the machine $T'$ outputs $0$,
otherwise $T'$ outputs $1$.
Since our halting oracle is entirely deterministic,
the output of $T'$ is entirely deterministic as well (and $T'$ always halts),
so
$\lambda_{T'}^O(0 \mid \epsilon) = 1$ or $\lambda_{T'}^O(1 \mid \epsilon) = 1$.
Therefore $O(T', \epsilon, 1/2) = 1$ or $O(T', \epsilon, 1/2) = 0$
because $O$ is reflective.
A precision of $1/3$ is enough to tell them apart,
hence $T'$ returns $0$ if $O(T', \epsilon, 1/2) = 1$ and
$T'$ returns $1$ if $O(T', \epsilon, 1/2) = 0$.
This is a contradiction.
\end{proof}

A similar argument can also be used to show that
reflective oracles are not computable.

\protected\def\dirtyhack{\ref*{thm:lc-reflective-oracle}}
\subsection{Proof of Theorem~\texorpdfstring{\dirtyhack}{6}}
\label{ssec:lc-reflective-oracle-proof}

The idea for the proof of \autoref{thm:lc-reflective-oracle} is to
construct an algorithm
that outputs an infinite series of \emph{partial oracles}
converging to a reflective oracle in the limit.

The set of queries is countable,
so we can assume that we have some computable enumeration of it:
\[
  \mathcal{T} \times \{ 0, 1 \}^* \times \mathbb{Q}
=: \{ q_1, q_2, \ldots \}
\]

\begin{definition}[$k$-Partial Oracle]
\label{def:k-partial-oracle}
A \emph{$k$-partial oracle} $\tilde O$ is function from the first $k$ queries
to the multiples of $2^{-k}$ in $[0, 1]$:
\[
\tilde O: \{ q_1, q_2, \ldots, q_k \} \to
\{ n 2^{-k} \mid 0 \leq n \leq 2^k \}
\]
\end{definition}

\begin{definition}[Approximating an Oracle]
\label{def:approx-oracle}
A $k$-partial oracle $\tilde O$ \emph{approximates} an oracle $O$ iff
$|O(q_i) - \tilde O(q_i)| \leq 2^{-k-1}$ for all $i \leq k$.
\end{definition}

Let $k \in \mathbb{N}$, let $\tilde O$ be a $k$-partial oracle, and
let $T \in \mathcal{T}$ be an oracle machine.
The machine $T^{\tilde O}$ that we get when
we run $T$ with the $k$-partial oracle $\tilde O$ is defined as follows
(this is with slight abuse of notation
since $k$ is taken to be understood implicitly).
\begin{enumerate}[1.]
\item Run $T$ for at most $k$ steps.
\item If $T$ calls the oracle on $q_i$ for $i \leq k$,
	\begin{enumerate}[(a)]
	\item return $1$ with probability $\tilde O(q_i) - 2^{-k-1}$,
	\item return $0$ with probability $1 - \tilde O(q_i) - 2^{-k-1}$, and
	\item halt otherwise.
	\end{enumerate}
\item If $T$ calls the oracle on $q_j$ for $j > k$, halt.
\end{enumerate}
Furthermore, we define $\lambda_T^{\tilde O}$ analogously to $\lambda_T^O$
as the distribution generated by the machine $T^{\tilde O}$.

\begin{lemma}\label{lem:approximating-an-oracle}
If a $k$-partial oracle $\tilde O$ approximates a reflective oracle $O$,
then $\lambda_T^O(1 \mid x) \geq \lambda_T^{\tilde O}(1 \mid x)$ and
$\lambda_T^O(0 \mid x) \geq \lambda_T^{\tilde O}(0 \mid x)$ for all
$x \in \{ 0, 1 \}^*$ and all $T \in \mathcal{T}$.
\end{lemma}
\begin{proof}
This follows from the definition of $T^{\tilde O}$:
when running $T$ with $\tilde O$ instead of $O$,
we can only lose probability mass.
If $T$ makes calls whose index is $> k$ or runs for more than $k$ steps,
then the execution is aborted
and no further output is generated.
If $T$ makes calls whose index $i \leq k$, then
$\tilde O(q_i) - 2^{-k-1} \leq O(q_i)$ since $\tilde O$ approximates $O$.
Therefore the return of the call $q_i$ is underestimated as well.
\end{proof}

\begin{definition}[$k$-Partially Reflective]
\label{def:k-partially-reflective}
A $k$-partial oracle $\tilde O$ is \emph{$k$-partially reflective} iff
for the first $k$ queries $(T, x, p)$
\begin{itemize}
\item $\lambda_T^{\tilde O}(1 \mid x) > p$ implies $\tilde O(T, x, p) = 1$, and
\item $\lambda_T^{\tilde O}(0 \mid x) > 1 - p$ implies $\tilde O(T, x, p) = 0$.
\end{itemize}
\end{definition}

It is important to note that we can check whether a $k$-partial oracle
is $k$-partially reflective in finite time by running all machines
$T$ from the first $k$ queries for $k$ steps and tallying up the probabilities
to compute $\lambda_T^{\tilde O}$.

\begin{lemma}\label{lem:partially-reflective}
If $O$ is a reflective oracle and
$\tilde O$ is a $k$-partial oracle that approximates $O$, then
$\tilde O$ is $k$-partially reflective.
\end{lemma}

\autoref{lem:partially-reflective} only holds
because we use semimeasures whose conditionals are lower semicomputable.

\begin{proof}
Assuming $\lambda_T^{\tilde O}(1 \mid x) > p$ we get
from \autoref{lem:approximating-an-oracle} that
$
     \lambda_T^O(1 \mid x)
\geq \lambda_T^{\tilde O}(1 \mid x)
>    p
$.
Thus $O(T, x, p) = 1$ because $O$ is reflective.
Since $\tilde O$ approximates $O$,
we get $1 = O(T, x, p) \leq \tilde O(T, x, p) + 2^{-k-1}$, and
since $\tilde O$ assigns values in a $2^{-k}$-grid,
it follows that $\tilde O(T, x, p) = 1$.
The second implication is proved analogously.
\end{proof}

\begin{definition}[Extending Partial Oracles]
\label{def:extending-partial-oracles}
A $k + 1$-partial oracle $\tilde O'$ \emph{extends}
a $k$-partial oracle $\tilde O$ iff
$|\tilde O(q_i) - \tilde O'(q_i)| \leq 2^{-k-1}$ for all $i \leq k$.
\end{definition}

\begin{lemma}
\label{lem:infinite-partial-oracles}
There is an infinite sequence of partial oracles $(\tilde O_k)_{k \in \mathbb{N}}$
such that for each $k$,
$\tilde O_k$ is a $k$-partially reflective $k$-partial oracle
and $\tilde O_{k+1}$ extends $\tilde O_k$.
\end{lemma}
\begin{proof}
By \autoref{thm:existence-reflective-oracles}
there is a reflective oracle $O$.
For every $k$, there is a canonical $k$-partial oracle $\tilde O_k$ that approximates $O$:
restrict $O$ to the first $k$ queries and for any such query $q$
pick the value in the $2^{-k}$-grid which is closest to $O(q)$.
By construction, $\tilde O_{k+1}$ extends $\tilde O_k$
and by \autoref{lem:partially-reflective}, each $\tilde O_k$ is $k$-partially reflective.
\end{proof}

\begin{lemma}\label{lem:extending-oracles}
If the $k+1$-partial oracle $\tilde O_{k+1}$ extends
the $k$-partial oracle $\tilde O_k$, then
$\lambda_T^{\tilde O_{k+1}}(1 \mid x) \geq \lambda_T^{\tilde O_k}(1 \mid x)$ and
$\lambda_T^{\tilde O_{k+1}}(0 \mid x) \geq \lambda_T^{\tilde O_k}(0 \mid x)$
for all $x \in \{ 0, 1 \}^*$ and all $T \in \mathcal{T}$.
\end{lemma}
\begin{proof}
$T^{\tilde O_{k+1}}$ runs for one more step than $T^{\tilde O_k}$,
can answer one more query and has increased oracle precision.
Moreover, since $\tilde O_{k+1}$ extends $\tilde O_k$,
we have $|\tilde O_{k+1}(q_i) - \tilde O_k(q_i)| \leq 2^{-k-1}$, and thus
$\tilde O_{k+1}(q_i) - 2^{-k-1} \geq \tilde O_k(q_i) - 2^{-k}$.
Therefore the success to answers to the oracle calls (case 2(a) and 2(b))
will not decrease in probability.
\end{proof}

Now everything is in place to state the algorithm
that constructs a reflective oracle in the limit.
It recursively traverses a tree of partial oracles.
The tree's nodes are the partial oracles;
level $k$ of the tree contains all $k$-partial oracles.
There is an edge in the tree from the $k$-partial oracle $\tilde O_k$ to
the $i$-partial oracle $\tilde O_i$ if and only if
$i = k + 1$ and $\tilde O_i$ extends $\tilde O_k$.

For every $k$, there are only finitely many $k$-partial oracles,
since they are functions from finite sets to finite sets.
In particular, there are exactly two $1$-partial oracles (so the search tree has two roots).
Pick one of them to start with, and proceed recursively as follows.
Given a $k$-partial oracle $\tilde O_k$,
there are finitely many $(k + 1)$-partial oracles that extend $\tilde O_k$
(finite branching of the tree).
Pick one that is $(k + 1)$-partially reflective
(which can be checked in finite time).
If there is no $(k + 1)$-partially reflective extension, backtrack.

By \autoref{lem:infinite-partial-oracles}
our search tree is infinitely deep and thus
the tree search does not terminate.
Moreover, it can backtrack to each level only a finite number of times
because at each level there is only a finite number of possible extensions.
Therefore the algorithm will produce an infinite sequence of partial oracles,
each extending the previous.
Because of finite backtracking, the output eventually stabilizes on a sequence
of partial oracles $\tilde O_1, \tilde O_2, \ldots$.
By the following lemma, this sequence converges to a reflective oracle,
which concludes the proof of \autoref{thm:lc-reflective-oracle}.

\begin{lemma}\label{lem:limit}
Let $\tilde O_1, \tilde O_2, \ldots$ be a sequence where
$\tilde O_k$ is a $k$-partially reflective $k$-partial oracle and
$\tilde O_{k+1}$ extends $\tilde O_k$ for all $k \in \mathbb{N}$.
Let $O := \lim_{k \to \infty} \tilde O_k$ be the pointwise limit.
Then
\begin{enumerate}[(a)]
\item $\lambda_T^{\tilde O_k}(1 \mid x) \to \lambda_T^O(1 \mid x)$ and
      $\lambda_T^{\tilde O_k}(0 \mid x) \to \lambda_T^O(0 \mid x)$ as $k \to \infty$
      for all $x \in \{ 0, 1 \}^*$ and all $T \in \mathcal{T}$, and
\item $O$ is a reflective oracle.
\end{enumerate}
\end{lemma}
\begin{proof}
First note that the pointwise limit must exists because
$|\tilde O_k(q_i) - \tilde O_{k+1}(q_i)| \leq 2^{-k-1}$
by \autoref{def:extending-partial-oracles}.
\begin{enumerate}[(a)]
\item Since $\tilde O_{k+1}$ extends $\tilde O_k$,
    each $\tilde O_k$ approximates $O$.
    Let $x \in \{ 0, 1 \}^*$ and $T \in \mathcal{T}$ and
    consider the sequence $a_k := \lambda_T^{\tilde O_k}(1 \mid x)$ for $k \in \mathbb{N}$.
    By \autoref{lem:extending-oracles},
    $a_k \leq a_{k+1}$, so the sequence is monotone increasing.
    By \autoref{lem:approximating-an-oracle},
    $a_k \leq \lambda_T^O(1 \mid x)$, so the sequence is bounded.
    Therefore it must converge.
    But it cannot converge to anything strictly below $\lambda_T^O(1 \mid x)$
    by the definition of $T^O$.
\item By definition, $O$ is an oracle; it remains to show that $O$ is reflective.
    Let $q_i = (T, x, p)$ be some query.
    If $p < \lambda_T^O(1 \mid x)$, then by (a)
    there is a $k$ large enough such that
    $p < \lambda_T^{\tilde O_t}(1 \mid x)$ for all $t \geq k$.
    For any $t \geq \max \{ k, i \}$,
    we have $\tilde O_t(T, x, p) = 1$
    since $\tilde O_t$ is $t$-partially reflective.
    Therefore $1 = \lim_{k \to \infty} \tilde O_k(T, x, p) = O(T, x, p)$.
    The case $1 - p < \lambda_T^O(0 \mid x)$ is analogous.
    \qedhere
\end{enumerate}
\end{proof}

\section{A Grain of Truth}
\label{sec:a-grain-of-truth}

\subsection{Notation}
\label{ssec:notation}

In reinforcement learning,
an agent interacts with an environment in cycles:
at time step $t$ the agent chooses an \emph{action} $a_t \in \A$ and
receives a \emph{percept} $e_t = (o_t, r_t) \in \E$
consisting of an \emph{observation} $o_t \in \O$
and a real-valued \emph{reward} $r_t \in \mathbb{R}$;
the cycle then repeats for $t + 1$.
A \emph{history} is an element of $\H$.
In this section,
we use $\ae \in \A \times \E$ to denote one interaction cycle,
and $\ae_{<t}$ to denote a history of length $t - 1$.

We fix a \emph{discount function}
$\gamma: \mathbb{N} \to \mathbb{R}$ with
$\gamma_t \geq 0$ and $\sum_{t=1}^\infty \gamma_t < \infty$.
The goal in reinforcement learning is
to maximize discounted rewards $\sum_{t=1}^\infty \gamma_t r_t$.
The \emph{discount normalization factor} is defined as
$\Gamma_t := \sum_{k=t}^\infty \gamma_k$.
The \emph{effective horizon} $H_t(\varepsilon)$ is a horizon
that is long enough to encompass all but an $\varepsilon$
of the discount function's mass:
\begin{equation}\label{eq:effective-horizon}
H_t(\varepsilon) := \min \{ k \mid \Gamma_{t+k} / \Gamma_t \leq \varepsilon \}
\end{equation}

A \emph{policy} is
a function $\pi: \H \to \Delta\A$
that maps a history $\ae_{<t}$ to
a distribution over actions taken after seeing this history.
The probability of taking action $a$ after history $\ae_{<t}$
is denoted with $\pi(a \mid \ae_{<t})$.
An \emph{environment} is a function $\nu: \H \times \A \to \Delta\E$
where $\nu(e \mid \ae_{<t}a_t)$ denotes
the probability of receiving the percept $e$
when taking the action $a_t$ after the history $\ae_{<t}$.
Together, a policy $\pi$ and an environment $\nu$ give rise to
a distribution $\nu^\pi$ over histories.
Throughout this paper, we make the following assumptions.

\begin{assumption}\label{ass:grl}
\begin{enumerate}[(a)]
\item \label{ass:bounded-rewards}
	Rewards are bounded between $0$ and $1$.
\item \label{ass:finite-actions-and-percepts}
	The set of actions $\A$ and the set of percepts $\E$
	are both finite.
\item \label{ass:gamma-computable}
	The discount function $\gamma$ and
	the discount normalization factor $\Gamma$ are computable.
\end{enumerate}
\end{assumption}

\begin{definition}[Value Function]
\label{def:discounted-value}
The \emph{value} of a policy $\pi$ in an environment $\nu$
given history $\ae_{<t}$ is defined recursively as
$V^\pi_\nu(\ae_{<t}) := \sum_{a \in \A} \pi(a \mid \ae_{<t}) V^\pi_\nu(\ae_{<t} a)$ and
\begin{align*}
  V^\pi_\nu&(\ae_{<t} a_t)
  := \\ &\frac{1}{\Gamma_t} \sum_{e_t \in \E}
      \nu(e_t \mid \ae_{<t} a_t)
      \Big( \gamma_t r_t +
      \Gamma_{t+1} V^\pi_\nu(\ae_{1:t}) \Big)
\end{align*}
if $\Gamma_t > 0$ and $V^\pi_\nu(\ae_{<t} a_t) := 0$ if $\Gamma_t = 0$.
The \emph{optimal value} is defined as
$V^*_\nu(\ae_{<t}) := \sup_\pi V^\pi_\nu(\ae_{<t})$.
\end{definition}

\begin{definition}[Optimal Policy]
\label{def:optimal-policy}
A policy $\pi$ is \emph{optimal in environment $\nu$ ($\nu$-optimal)} iff
for all histories $\ae_{<t} \in \H$
the policy $\pi$ attains the optimal value:
$V^\pi_\nu(\ae_{<t}) = V^*_\nu(\ae_{<t})$.
\end{definition}

We assumed that
the discount function is summable,
rewards are bounded
(\assref{ass:bounded-rewards}), and
actions and percepts spaces are both finite
(\assref{ass:finite-actions-and-percepts}).
Therefore an optimal deterministic policy exists
for every environment~\cite[Thm.\ 10]{LH:2014discounting}.

\subsection{Reflective Bayesian Agents}
\label{ssec:reflective-Bayesian-agents}

Fix $O$ to be a reflective oracle.
From now on,
we assume that the action space $\A := \{ \alpha, \beta \}$ is binary.
We can treat computable measures over binary strings as environments:
the environment $\nu$ corresponding to
a probabilistic Turing machine $T \in \mathcal{T}$
is defined by
\[
   \nu(e_t \mid \ae_{<t} a_t)
:= \overline\lambda_T^O(y \mid x)
 = \prod_{i=1}^k \overline\lambda_T^O(y_i \mid x y_1 \ldots y_{i-1})
\]
where $y_{1:k}$ is a binary encoding of $e_t$ and
$x$ is a binary encoding of $\ae_{<t} a_t$.
The actions $a_{1:\infty}$ are only \emph{contextual},
and not part of the environment distribution.
We define
\[
   \nu(e_{<t} \mid a_{<t})
:= \prod_{k=1}^{t-1} \nu(e_k \mid \ae_{<k}).
\]

Let $T_1, T_2, \ldots$ be an enumeration of
all probabilistic Turing machines in $\mathcal{T}$.
We define the \emph{class of reflective environments}
\[
   \Mrefl^O
:= \left\{ \overline\lambda_{T_1}^O, \overline\lambda_{T_2}^O, \ldots \right\}.
\]
This is the class of all environments computable on
a probabilistic Turing machine with reflective oracle $O$,
that have been completed from semimeasures to measures using $O$.

Analogously to AIXI~\cite{Hutter:2005},
we define a Bayesian mixture over the class $\Mrefl^O$.
Let $w \in \Delta\Mrefl^O$ be
a lower semicomputable prior probability distribution on $\Mrefl^O$.
Possible choices for the prior include the \emph{Solomonoff prior}
$w\big(\overline\lambda_T^O\big) := 2^{-K(T)}$, where $K(T)$ denotes
the length of the shortest input to some universal Turing machine that encodes $T$~\cite{Solomonoff:1978}.%
\footnote{%
Technically, the lower semicomputable prior $2^{-K(T)}$
is only a semidistribution because it does not sum to $1$.
This turns out to be unimportant.
}
We define the corresponding Bayesian mixture
\begin{equation}\label{eq:Bayes-mixture}
   \xi(e_t \mid \ae_{<t} a_t)
:= \sum_{\nu \in \Mrefl^O} w(\nu \mid \ae_{<t}) \nu(e_t \mid \ae_{<t} a_t)
\end{equation}
where $w(\nu \mid \ae_{<t})$ is the (renormalized) posterior,
\begin{equation}\label{eq:posterior}
   w(\nu \mid \ae_{<t})
:= w(\nu) \frac{\nu(e_{<t} \mid a_{<t})}{\overline\xi(e_{<t} \mid a_{<t})}.
\end{equation}
The mixture $\xi$ is lower semicomputable on an oracle Turing machine
because the posterior $w(\,\cdot \mid \ae_{<t})$ is lower semicomputable.
Hence there is an oracle machine $T$ such that $\xi = \lambda_T^O$.
We define its completion $\overline\xi := \overline\lambda_T^O$
as the completion of $\lambda_T^O$.
This is the distribution that is used to compute the posterior.
There are no cyclic dependencies since
$\overline\xi$ is called on the shorter history $\ae_{<t}$.
We arrive at the following statement.

\begin{proposition}[Bayes is in the Class]
\label{prop:Bayes-is-in-the-class}
$\overline\xi \in \Mrefl^O$.
\end{proposition}

Moreover, since $O$ is reflective,
we have that $\overline\xi$ dominates all environments $\nu \in \Mrefl^O$:
\begingroup
\allowdisplaybreaks
\begin{align*}
&\phantom{=}~\; \overline\xi(e_{1:t} \mid a_{1:t}) \\
&=    \overline\xi(e_t \mid \ae_{<t} a_t) \overline\xi(e_{<t} \mid a_{<t}) \\
&\geq \xi(e_t \mid \ae_{<t} a_t) \overline\xi(e_{<t} \mid a_{<t}) \\
&=    \overline\xi(e_{<t} \mid a_{<t})
      \sum_{\nu \in \Mrefl^O} w(\nu \mid \ae_{<t}) \nu(e_t \mid \ae_{<t} a_t) \\
&=    \overline\xi(e_{<t} \mid a_{<t}) \sum_{\nu \in \Mrefl^O} w(\nu) \frac{\nu(e_{<t} \mid a_{<t})}{\overline\xi(e_{<t} \mid a_{<t})} \nu(e_t \mid \ae_{<t} a_t) \\
&=    \sum_{\nu \in \Mrefl^O} w(\nu) \nu(e_{1:t} \mid a_{1:t}) \\
&\geq w(\nu) \nu(e_{1:t} \mid a_{1:t})
\end{align*}
\endgroup
This property is crucial for on-policy value convergence.

\begin{lemma}[{On-Policy Value Convergence~\cite[Thm.~5.36]{Hutter:2005}}]
\label{lem:on-policy-value-convergence}
For any policy $\pi$ and any environment $\mu \in \Mrefl^O$ with $w(\mu) > 0$,
\[
V^\pi_\mu(\ae_{<t}) - V^\pi_{\overline\xi}(\ae_{<t}) \to 0
\text{ $\mu^\pi$-almost surely as $t \to \infty$}.
\]
\end{lemma}

\subsection{Reflective-Oracle-Computable Policies}
\label{ssec:reflective-oracle-computable-policies}

This subsection is dedicated to the following result
that was previously stated but not proved in \cite[Alg.~6]{FST:2015}.
It contrasts results on
arbitrary semicomputable environments where optimal policies are
not limit computable~\cite[Sec.~4]{LH:2015computability}.

\begin{theorem}[Optimal Policies are Oracle Computable]
\label{thm:optimal-policies-are-oracle-computable}
For every $\nu \in \Mrefl^O$,
there is a $\nu$-optimal (stochastic) policy $\pi^*_\nu$
that is reflective-oracle-computable.
\end{theorem}

Note that even though deterministic optimal policies always exist,
those policies are typically not reflective-oracle-computable.

To prove \autoref{thm:optimal-policies-are-oracle-computable}
we need the following lemma.

\begin{lemma}[Reflective-Oracle-Computable Optimal Value Function]
\label{lem:optimal-value-reflective-oracle-computable}
For every environment $\nu \in \Mrefl^O$
the optimal value function $V^*_\nu$ is reflective-oracle-computable.
\end{lemma}
\begin{proof}
This proof follows the proof of \cite[Cor.~13]{LH:2015computability}.
We write the optimal value explicitly as
\begin{equation}\label{eq:V-explicit}
  V^*_\nu(\ae_{<t})
= \frac{1}{\Gamma_t} \lim_{m \to \infty} \expectimax{\ae_{t:m}}\;
    \sum_{k=t}^m \gamma_k r_k \prod_{i=t}^k \nu(e_i \mid \ae_{<i}),
\end{equation}
where $\expectimax{}$ denotes the expectimax operator:
\[
   \expectimax{\ae_{t:m}}
:= \max_{a_t \in \A} \sum_{e_t \in \E} \ldots \max_{a_m \in \A} \sum_{e_m \in \E}
\]
For a fixed $m$,
all involved quantities are reflective-oracle-computable.
Moreover, this quantity is monotone increasing in $m$ and
the tail sum from $m+1$ to $\infty$ is bounded by $\Gamma_{m+1}$
which is computable according to \assref{ass:gamma-computable}
and converges to $0$ as $m \to \infty$.
Therefore we can enumerate all rationals above and below $V^*_\nu$.
\end{proof}

\begin{proof}[Proof of \autoref{thm:optimal-policies-are-oracle-computable}]
According to \autoref{lem:optimal-value-reflective-oracle-computable}
the optimal value function $V^*_\nu$ is reflective-oracle-computable.
Hence there is a probabilistic Turing machine $T$ such that
\[
  \lambda_T^O(1 \mid \ae_{<t})
= \big( V^*_\nu(\ae_{<t} \alpha) - V^*_\nu(\ae_{<t} \beta) + 1 \big) / 2.
\]
We define a policy $\pi$ that takes
action $\alpha$ if $O(T, \ae_{<t}, 1/2) = 1$ and
action $\beta$ if $O(T, \ae_{<t}, 1/2) = 0$.
(This policy is stochastic because the answer of the oracle $O$ is stochastic.)

It remains to show that $\pi$ is a $\nu$-optimal policy.
If $V^*_\nu(\ae_{<t} \alpha) > V^*_\nu(\ae_{<t} \beta)$,
then $\lambda_T^O(1 \mid \ae_{<t}) > 1/2$,
thus $O(T, \ae_{<t}, 1/2) = 1$ since $O$ is reflective,
and hence $\pi$ takes action $\alpha$.
Conversely, if $V^*_\nu(\ae_{<t} \alpha) < V^*_\nu(\ae_{<t} \beta)$,
then $\lambda_T^O(1 \mid \ae_{<t}) < 1/2$,
thus $O(T, \ae_{<t}, 1/2) = 0$ since $O$ is reflective,
and hence $\pi$ takes action $\beta$.
Lastly, if $V^*_\nu(\ae_{<t} \alpha) = V^*_\nu(\ae_{<t} \beta)$,
then both actions are optimal and
thus it does not matter which action is returned by policy $\pi$.
(This is the case where the oracle may randomize.)
\end{proof}

\subsection{Solution to the Grain of Truth Problem}
\label{ssec:solution-to-grain-of-truth-problem}

Together,
\autoref{prop:Bayes-is-in-the-class} and
\autoref{thm:optimal-policies-are-oracle-computable}
provide the necessary ingredients to solve the grain of truth problem.

\begin{corollary}[Solution to the Grain of Truth Problem]
\label{cor:solution-to-grain-of-truth}
For every lower semicomputable prior $w \in \Delta\Mrefl^O$
the Bayes-optimal policy $\pi^*_{\overline\xi}$ is reflective-oracle-computable
where $\xi$ is the Bayes-mixture corresponding to $w$
defined in \eqref{eq:Bayes-mixture}.
\end{corollary}
\begin{proof}
From \autoref{prop:Bayes-is-in-the-class}
and \autoref{thm:optimal-policies-are-oracle-computable}.
\end{proof}

Hence the environment class $\Mrefl^O$ contains
any reflective-oracle-computable modification
of the Bayes-optimal policy $\pi^*_{\overline\xi}$.
In particular,
this includes computable multi-agent environments
that contain other Bayesian agents over the class $\Mrefl^O$.
So any Bayesian agent over the class $\Mrefl^O$ has a grain of truth
even though the environment may contain other Bayesian agents
\emph{of equal power}.
We proceed to sketch the implications for multi-agent environments
in the next section.

\section{Multi-Agent Environments}
\label{sec:multi-agent-environments}

This section summarizes our results for multi-agent systems.
The proofs can be found in \cite{Leike:2016}.

\subsection{Setup}
\label{ssec:multi-agent-setup}

In a \emph{multi-agent environment}
there are $n$ agents
each taking sequential actions from the finite action space $\A$.
In each time step $t = 1, 2, \ldots$,
the environment receives action $a_t^i$ from agent $i$ and outputs
$n$ percepts $e_t^1, \ldots, e_t^n \in \E$, one for each agent.
Each percept $e_t^i = (o_t^i, r_t^i)$ contains
an observation $o_t^i$ and a reward $r_t^i \in [0, 1]$.
Importantly, agent $i$ only sees
its own action $a_t^i$ and its own percept $e_t^i$
(see \autoref{fig:multi-agent-model}).
We use the shorthand notation $a_t := (a_t^1, \ldots, a_t^n)$ and
$e_t := (e_t^1, \ldots, e_t^n)$ and denote
$\ae_{<t}^i = a_1^i e_1^i \ldots a_{t-1}^i e_{t-1}^i$ and
$\ae_{<t} = a_1 e_1 \ldots a_{t-1} e_{t-1}$.

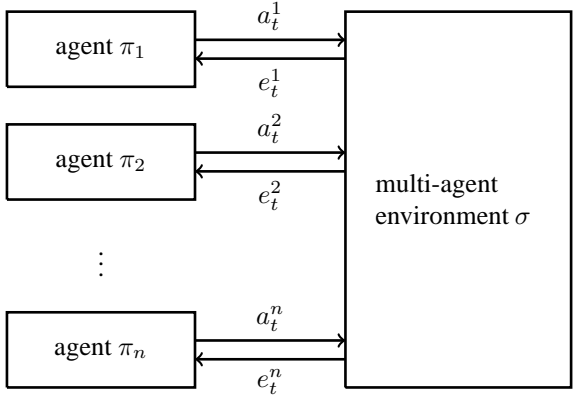
\begin{figure}[t]
\begin{center}
\begin{tikzpicture}[scale=0.25,line width=1pt] 
\draw (0,16) -- (10,16) -- (10,20) -- (0,20) -- (0,16);
\node at (5,18) {agent $\pi_1$};

\draw (0,10) -- (10,10) -- (10,14) -- (0,14) -- (0,10);
\node at (5,12) {agent $\pi_2$};

\node at (5,7) {\vdots};

\draw (0,0) -- (10,0) -- (10,4) -- (0,4) -- (0,0);
\node at (5,2) {agent $\pi_n$};

\draw (18,0) -- (30,0) -- (30,20) -- (18,20) -- (18,0);
\node at (24,10) {\begin{minipage}{22mm}
multi-agent \\ environment $\sigma$
\end{minipage}};

\draw[->] (10,18.5) to node[above] {$a_t^1$} (18,18.5);
\draw[<-] (10,17.5) to node[below] {$e_t^1$} (18,17.5);
\draw[->] (10,12.5) to node[above] {$a_t^2$} (18,12.5);
\draw[<-] (10,11.5) to node[below] {$e_t^2$} (18,11.5);
\draw[->] (10,2.5) to node[above] {$a_t^n$} (18,2.5);
\draw[<-] (10,1.5) to node[below] {$e_t^n$} (18,1.5);
\end{tikzpicture}
\end{center}
\caption[The multi-agent model]{%
Agents $\pi_1, \ldots, \pi_n$ interacting in a multi-agent environment.
}
\label{fig:multi-agent-model}
\end{figure}

We define a multi-agent environment as a function
\[
\sigma: (\A^n \times \E^n)^* \times \A^n \to \Delta(\E^n).
\]
The agents are given by $n$ policies $\pi_1, \ldots, \pi_n$ where
$\pi_i: (\A \times \E)^* \to \Delta \A$.
Together they specify the \emph{history distribution}
\begin{align*}
    \sigma^{\pi_{1:n}}(\epsilon) :&= 1 \\
    \sigma^{\pi_{1:n}}(\ae_{1:t})
:&= \sigma^{\pi_{1:n}}(\ae_{<t} a_t) \sigma(e_t \mid \ae_{<t} a_t) \\
    \sigma^{\pi_{1:n}}(\ae_{<t} a_t)
:&= \sigma^{\pi_{1:n}}(\ae_{<t}) \prod_{i=1}^n \pi_i(a_t^i \mid \ae_{<t}^i).
\end{align*}
Each agent $i$ acts in a \emph{subjective environment} $\sigma_i$
given by joining
the multi-agent environment $\sigma$
with the policies $\pi_1, \ldots, \pi_{i-1}, \pi_{i+1}, \ldots, \pi_n$
by marginalizing over the histories that $\pi_i$ does not see.
Together with policy $\pi_i$,
the environment $\sigma_i$ yields a distribution over the histories of agent $i$
\[
   \sigma_i^{\pi_i}(\ae_{<t}^i)
:= \sum_{\ae_{<t}^j, j \neq i} \sigma^{\pi_{1:n}}(\ae_{<t}).
\]
We get the definition of the subjective environment $\sigma_i$ with the identity
$\sigma_i(e_t^i \mid \ae_{<t}^i a_t^i)
:= \sigma_i^{\pi_i}(e_t^i \mid \ae_{<t}^i a_t^i)$.
It is crucial to note that
the subjective environment $\sigma_i$ and the policy $\pi_i$
are ordinary environments and policies,
so we can use the formalism from \autoref{sec:a-grain-of-truth}.

Our definition of a multi-agent environment is very general
and encompasses most of game theory.
It allows for cooperative, competitive, and mixed games;
infinitely repeated games or any (infinite-length) extensive form games
with finitely many players.

The policy $\pi_i$
is an \emph{$\varepsilon$-best response} after history $\ae_{<t}^i$ iff
\[
  V^*_{\sigma_i}(\ae_{<t}^i) - V^{\pi_i}_{\sigma_i}(\ae_{<t}^i)
< \varepsilon.
\]

If at some time step $t$,
all agents' policies are $\varepsilon$-best responses,
we have an \emph{$\varepsilon$-Nash equilibrium}.
The property of multi-agent systems that is analogous to
asymptotic optimality is convergence to an $\varepsilon$-Nash equilibrium.

\subsection{Informed Reflective Agents}
\label{ssec:informed-reflective-agents}

Let $\sigma$ be a multi-agent environment
and let $\pi^*_{\sigma_1}, \ldots \pi^*_{\sigma_n}$ be such that
for each $i$ the policy $\pi^*_{\sigma_i}$ is an optimal policy
in agent $i$'s subjective environment $\sigma_i$.
At first glance this seems ill-defined:
The subjective environment $\sigma_i$ depends on each other policy
$\pi^*_{\sigma_j}$ for $j \neq i$,
which depends on the subjective environment $\sigma_j$,
which in turn depends on the policy $\pi^*_{\sigma_i}$.
However, this circular definition actually has a well-defined solution.

\begin{theorem}[Optimal Multi-Agent Policies]
\label{thm:informed-reflective-agents}
For any reflective-oracle-computable multi-agent environment $\sigma$,
the optimal policies $\pi^*_{\sigma_1}, \ldots, \pi^*_{\sigma_n}$
exist and are reflective-oracle-computable.
\end{theorem}

Note the strength of \autoref{thm:informed-reflective-agents}:
each of the policies $\pi^*_{\sigma_i}$ is acting optimally
\emph{given the knowledge of everyone else's policies}.
Hence optimal policies play $0$-best responses by definition,
so if every agent is playing an optimal policy, we have a Nash equilibrium.
Moreover, this Nash equilibrium is also a \emph{subgame perfect} Nash equilibrium,
because each agent also acts optimally on the counterfactual histories
that do not end up being played.
In other words,
\autoref{thm:informed-reflective-agents}
states the existence and reflective-oracle-computability
of a subgame perfect Nash equilibrium
in any reflective-oracle-computable multi-agent environment.
From \autoref{thm:lc-reflective-oracle} we then get that
these subgame perfect Nash equilibria are limit computable.

\begin{corollary}[Solution to Computable Multi-Agent Environments]
\label{cor:optimal-multi-agent-policies}
For any computable multi-agent environment $\sigma$,
the optimal policies $\pi^*_{\sigma_1}, \ldots, \pi^*_{\sigma_n}$
exist and are limit computable.
\end{corollary}

\subsection{Learning Reflective Agents}
\label{ssec:learning-reflective-agents}

Since our class $\Mrefl^O$ solves the grain of truth problem,
the result by Kalai and Lehrer~\cite{KL:1993} immediately implies that
for any Bayesian agents $\pi_1, \ldots, \pi_n$
interacting in an infinitely repeated game and
for all $\varepsilon > 0$ and all $i \in \{ 1, \ldots, n \}$
there is almost surely a $t_0 \in \mathbb{N}$ such that for all $t \geq t_0$
the policy $\pi_i$ is an $\varepsilon$-best response.
However, this hinges on the important fact that
every agent has to know the game and
also that all other agents are Bayesian agents.
Otherwise the convergence to an $\varepsilon$-Nash equilibrium may fail,
as illustrated by the following example.

At the core of the following construction is a \emph{dogmatic prior}%
~\cite[Sec.~3.2]{LH:2015priors}.
A dogmatic prior assigns very high probability
to going to hell (reward $0$ forever)
if the agent deviates from a given computable policy $\pi$.
For a Bayesian agent it is thus only worth deviating from the policy $\pi$
if the agent thinks that the prospects of following $\pi$ are very poor already.
This implies that
for general multi-agent environments and
without additional assumptions on the prior,
we cannot prove any meaningful convergence result about Bayesian agents
acting in an unknown multi-agent environment.

\begin{example}[Reflective Bayesians Playing Matching Pennies]
\label{ex:reflective-Bayesians-playing-matching-pennies}
In the game of \emph{matching pennies} there are two agents ($n = 2$),
and two actions $\A = \{ \alpha, \beta \}$
representing the two sides of a penny.
In each time step
agent $1$ wins if the two actions are identical and
agent $2$ wins if the two actions are different.
The payoff matrix is as follows.
\begin{center}
\begin{tabular}{l|cc}
         & $\alpha$ & $\beta$ \\
\hline
$\alpha$ & 1,0      & 0,1 \\
$\beta$  & 0,1      & 1,0
\end{tabular}
\end{center}
We use $\E = \{ 0, 1 \}$ to be the set of rewards
(observations are vacuous) and define the multi-agent environment $\sigma$
to give reward $1$ to agent $1$ iff $a_t^1 = a_t^2$ ($0$ otherwise) and
reward $1$ to agent $2$ iff $a_t^1 \neq a_t^2$ ($0$ otherwise).
Note that neither agent knows a priori that they are playing matching pennies,
nor that they are playing an infinite repeated game with one other player.

Let $\pi_1$ be the policy that takes the action sequence
$(\alpha \alpha \beta)^\infty$ and
let $\pi_2 := \pi_\alpha$ be the policy that always takes action $\alpha$.
The average reward of policy $\pi_1$ is $2/3$ and
the average reward of policy $\pi_2$ is $1/3$.
Let $\xi$ be a universal mixture \eqref{eq:Bayes-mixture}.
By \autoref{lem:on-policy-value-convergence},
$V^{\pi_1}_{\overline\xi} \to c_1 \approx 2/3$ and
$V^{\pi_2}_{\overline\xi} \to c_2 \approx 1/3$
almost surely
when following policies $(\pi_1, \pi_2)$.
Therefore there is an $\varepsilon > 0$ such that
$V^{\pi_1}_{\overline\xi} > \varepsilon$ and
$V^{\pi_2}_{\overline\xi} > \varepsilon$
for all time steps.
Now we can apply \cite[Thm.~7]{LH:2015priors} to conclude that
there are (dogmatic) mixtures $\xi_1'$ and $\xi_2'$ such that
$\pi^*_{\xi_1'}$ always follows policy $\pi_1$ and
$\pi^*_{\xi_2'}$ always follows policy $\pi_2$.
This does not converge to a ($\varepsilon$-)Nash equilibrium.
\end{example}

A policy $\pi$ is
\emph{asymptotically optimal in mean in an environment class $\M$}
iff for all $\mu \in \M$
\begin{equation}\label{eq:asymptotic-optimality}
\EE_\mu^\pi \big[ V^*_\mu(\ae_{<t}) - V^\pi_\mu(\ae_{<t}) \big] \to 0
\text{ as $t \to \infty$}
\end{equation}
where $\EE_\mu^\pi$ denotes the expectation with respect to
the probability distribution $\mu^\pi$ over histories
generated by policy $\pi$ acting in environment $\mu$.

Asymptotic optimality stands out because it is currently the only known
nontrivial objective notion of optimality in general reinforcement learning%
~\cite{LH:2015priors}.

The following theorem is the main convergence result.
It states that
for asymptotically optimal agents
we get convergence to $\varepsilon$-Nash equilibria
in any reflective-oracle-computable multi-agent environment.

\begin{theorem}[Convergence to Equilibrium]
\label{thm:convergence-to-equilibrium}
Let $\sigma$ be an reflective-oracle-computable multi-agent environment and
let $\pi_1, \ldots, \pi_n$ be reflective-oracle-computable policies
that are asymptotically optimal in mean in the class $\Mrefl^O$.
Then for all $\varepsilon > 0$ and all $i \in \{ 1, \ldots, n \}$
the $\sigma^{\pi_{1:n}}$-probability that
the policy $\pi_i$ is an $\varepsilon$-best response
converges to $1$ as $t \to \infty$.
\end{theorem}

In contrast to \autoref{thm:informed-reflective-agents}
which yields policies that play a subgame perfect equilibrium,
this is not the case for \autoref{thm:convergence-to-equilibrium}:
the agents typically do not learn to predict off-policy and
thus will generally not play $\varepsilon$-best responses
in the counterfactual histories that they never see.
This weaker form of equilibrium is unavoidable
if the agents do not know the environment because
it is impossible to learn the parts that they do not interact with.

Together with \autoref{thm:lc-reflective-oracle} and
the asymptotic optimality of the Thompson sampling policy%
~\cite[Thm.~4]{LLOH:2016Thompson} that is reflective-oracle computable
we get the following corollary.

\begin{corollary}[Convergence to Equilibrium]
\label{cor:convergence-to-equilibrium}
There are limit computable policies $\pi_1, \ldots, \pi_n$ such that
for any computable multi-agent environment $\sigma$ and
for all $\varepsilon > 0$ and all $i \in \{ 1, \ldots, n \}$
the $\sigma^{\pi_{1:n}}$-probability that
the policy $\pi_i$ is an $\varepsilon$-best response
converges to $1$ as $t \to \infty$.
\end{corollary}

\section{Discussion}
\label{sec:discussion}

This paper introduced
the class of all reflective-oracle-computable environments $\Mrefl^O$.
This class solves the grain of truth problem because
it contains (any computable modification of)
Bayesian agents defined over $\Mrefl^O$:
the optimal agents and Bayes-optimal agents
over the class are all reflective-oracle-computable~%
(\autoref{thm:optimal-policies-are-oracle-computable} and
\autoref{cor:solution-to-grain-of-truth}).

If the environment is unknown,
then a Bayesian agent may end up playing suboptimally~%
(\autoref{ex:reflective-Bayesians-playing-matching-pennies}).
However, if each agent uses a policy that is asymptotically optimal in mean
(such as the Thompson sampling policy~\cite{LLOH:2016Thompson})
then for every $\varepsilon > 0$
the agents converge to an $\varepsilon$-Nash equilibrium~%
(\autoref{thm:convergence-to-equilibrium} and
\autoref{cor:convergence-to-equilibrium}).

Our solution to the grain of truth problem is purely theoretical.
However, \autoref{thm:lc-reflective-oracle} shows that
our class $\Mrefl^O$ allows for computable approximations.
This suggests that practical approaches can be derived from this result,
and reflective oracles have already seen
applications in one-shot games~\cite{FTC:2015reflection}.

\subsubsection*{Acknowledgements}

We thank Marcus Hutter and Tom Everitt for valuable comments.


\bibliographystyle{alpha}
\bibliography{ai}

\begin{thebibliography}{LLOH16}

\bibitem[FST15]{FST:2015}
Benja Fallenstein, Nate Soares, and Jessica Taylor.
\newblock Reflective variants of {S}olomonoff induction and {AIXI}.
\newblock In {\em Artificial General Intelligence}. Springer, 2015.

\bibitem[FTC15a]{FTC:2015reflectionx}
Benja Fallenstein, Jessica Taylor, and Paul~F Christiano.
\newblock Reflective oracles: A foundation for classical game theory.
\newblock Technical report, Machine Intelligence Research Institute, 2015.
\newblock \url{http://arxiv.org/abs/1508.04145}.

\bibitem[FTC15b]{FTC:2015reflection}
Benja Fallenstein, Jessica Taylor, and Paul~F Christiano.
\newblock Reflective oracles: A foundation for game theory in artificial
  intelligence.
\newblock In {\em Logic, Rationality, and Interaction}, pages 411--415.
  Springer, 2015.

\bibitem[FY01]{FY:2001impossibility}
Dean~P Foster and H~Peyton Young.
\newblock On the impossibility of predicting the behavior of rational agents.
\newblock {\em Proceedings of the National Academy of Sciences},
  98(22):12848--12853, 2001.

\bibitem[Hut05]{Hutter:2005}
Marcus Hutter.
\newblock {\em Universal Artificial Intelligence}.
\newblock Springer, 2005.

\bibitem[Hut09]{Hutter:2009open}
Marcus Hutter.
\newblock Open problems in universal induction \& intelligence.
\newblock {\em Algorithms}, 3(2):879--906, 2009.

\bibitem[KL93]{KL:1993}
Ehud Kalai and Ehud Lehrer.
\newblock Rational learning leads to {N}ash equilibrium.
\newblock {\em Econometrica}, pages 1019--1045, 1993.

\bibitem[Kle52]{Kleene:1952}
Stephen~Cole Kleene.
\newblock {\em Introduction to Metamathematics}.
\newblock Wolters-Noordhoff Publishing, 1952.

\bibitem[Lei16]{Leike:2016}
Jan Leike.
\newblock {\em Nonparametric General Reinforcement Learning}.
\newblock PhD thesis, Australian National University, 2016.

\bibitem[LH14]{LH:2014discounting}
Tor Lattimore and Marcus Hutter.
\newblock General time consistent discounting.
\newblock {\em Theoretical Computer Science}, 519:140--154, 2014.

\bibitem[LH15a]{LH:2015priors}
Jan Leike and Marcus Hutter.
\newblock Bad universal priors and notions of optimality.
\newblock In {\em Conference on Learning Theory}, pages 1244--1259, 2015.

\bibitem[LH15b]{LH:2015computability}
Jan Leike and Marcus Hutter.
\newblock On the computability of {AIXI}.
\newblock In {\em Uncertainty in Artificial Intelligence}, pages 464--473,
  2015.

\bibitem[LH15c]{LH:2015computability2}
Jan Leike and Marcus Hutter.
\newblock On the computability of {S}olomonoff induction and knowledge-seeking.
\newblock In {\em Algorithmic Learning Theory}, pages 364--378, 2015.

\bibitem[LLOH16]{LLOH:2016Thompson}
Jan Leike, Tor Lattimore, Laurent Orseau, and Marcus Hutter.
\newblock {T}hompson sampling is asymptotically optimal in general
  environments.
\newblock In {\em Uncertainty in Artificial Intelligence}, 2016.

\bibitem[LV08]{LV:2008}
Ming Li and Paul M.~B. Vitányi.
\newblock {\em An Introduction to {K}olmogorov Complexity and Its
  Applications}.
\newblock Texts in Computer Science. Springer, 3rd edition, 2008.

\bibitem[Nac97]{Nachbar:1997}
John~H Nachbar.
\newblock Prediction, optimization, and learning in repeated games.
\newblock {\em Econometrica}, 65(2):275--309, 1997.

\bibitem[Nac05]{Nachbar:2005}
John~H Nachbar.
\newblock Beliefs in repeated games.
\newblock {\em Econometrica}, 73(2):459--480, 2005.

\bibitem[Ors13]{Orseau:2013}
Laurent Orseau.
\newblock Asymptotic non-learnability of universal agents with computable
  horizon functions.
\newblock {\em Theoretical Computer Science}, 473:149--156, 2013.

\bibitem[SLB09]{SLB:2009}
Yoav Shoham and Kevin Leyton-Brown.
\newblock {\em Multiagent Systems: Algorithmic, Game-Theoretic, and Logical
  Foundations}.
\newblock Cambridge University Press, 2009.

\bibitem[Sol78]{Solomonoff:1978}
Ray Solomonoff.
\newblock Complexity-based induction systems: Comparisons and convergence
  theorems.
\newblock {\em IEEE Transactions on Information Theory}, 24(4):422--432, 1978.

\end{thebibliography}

\cleardoublepage
\appendix
\section{List of Notation}
\label{app:notation}

\begin{xtabular*}{\linewidth}{lp{60mm}}
$:=$
	& defined to be equal \\
$\mathbb{N}$
	& the natural numbers, starting with $0$ \\
$\mathbb{Q}$
	& the rational numbers \\
$\mathbb{R}$
	& the real numbers \\
$t$
	& (current) time step, $t \in \mathbb{N}$ \\
$k, n, i$
	& time steps, natural numbers \\
$p$
	& a rational number \\
$\X^*$
	& the set of all finite strings over the alphabet $\X$ \\
$\X^\infty$
	& the set of all infinite strings over the alphabet $\X$ \\
$\X^\sharp$
	& the set of all finite and infinite strings over the alphabet $\X$ \\
$O$
	& a reflective oracle \\
$\tilde O$
	& a partial oracle \\
$q$
	& a query to a reflective oracle \\
$\mathcal{T}$
	& the set of all probabilistic Turing machines that can query an oracle \\
$T, T'$
	& probabilistic Turing machines that can query an oracle,
	$T, T' \in \mathcal{T}$ \\
$K(x)$
	& the Kolmogorov complexity of a string $x$ \\
$\lambda_T$
	& the semimeasure corresponding to the probabilistic Turing machine $T$ \\
$\lambda_T^O$
	& the semimeasure corresponding to the probabilistic Turing machine $T$
	with reflective oracle $O$ \\
$\overline\lambda_T^O$
	& the completion of $\lambda_T^O$
	into a measure using the reflective oracle $O$ \\
$\A$
	& the finite set of possible actions \\
$\O$
	& the finite set of possible observations \\
$\E$
	& the finite set of possible percepts,
	$\E \subset \O \times \mathbb{R}$ \\
$\alpha, \beta$
	& two different actions, $\alpha, \beta \in \A$ \\
$a_t$
	& the action in time step $t$ \\
$o_t$
	& the observation in time step $t$ \\
$r_t$
	& the reward in time step $t$, bounded between $0$ and $1$ \\
$e_t$
	& the percept in time step $t$, we use $e_t = (o_t, r_t)$ implicitly \\
$\ae_{<t}$
	& the first $t - 1$ interactions,
	$a_1 e_1 a_2 e_2 \ldots a_{t-1} e_{t-1}$
	(a history of length $t - 1$) \\
$\epsilon$
	& the empty string/the history of length $0$ \\
$\varepsilon$
	& a small positive real number \\
$\gamma$
	& the discount function $\gamma: \mathbb{N} \to \mathbb{R}_{\geq0}$ \\
$\Gamma_t$
	& a discount normalization factor,
	$\Gamma_t := \sum_{k=t}^\infty \gamma_k$ \\
$\nu, \mu$
	& environments/semimeasures \\
$\sigma$
	& multi-agent environment \\
$\sigma^{\pi_{1:n}}$
	& history distribution induced by policies $\pi_1, \ldots, \pi_n$
	  acting in the multi-agent environment $\sigma$ \\
$\sigma_i$
	& subjective environment of agent $i$ \\
$\pi$
	& a policy, $\pi: \H \to \A$ \\
$\pi^*_\nu$
	& an optimal policy for environment $\nu$ \\
$V^\pi_\nu$
	& the $\nu$-expected value of the policy $\pi$ \\
$V^*_\nu$
	& the optimal value in environment $\nu$ \\
$\M$
	& a countable class of environments \\
$\Mrefl^O$
	& the class of all reflective-oracle-computable environments \\
$w$
	& a universal prior, $w \in \Delta\Mrefl^O$ \\
$\xi$
	& the universal mixture over all environments $\Mrefl^O$, a semimeasure \\
$\overline\xi$
	& the completion of $\lambda_T^O$
	into a measure using the reflective oracle $O$ \\
\end{xtabular*}

\end{document}